\def\year{2020}\relax
\documentclass[letterpaper]{article} 
\usepackage{aaai20}  
\usepackage{times}  
\usepackage{helvet} 
\usepackage{courier}  
\usepackage[hyphens]{url}  
\usepackage{graphicx} 
\usepackage{graphicx}  
\usepackage{subcaption}
\usepackage{amssymb}
\usepackage{algorithm}
\usepackage{algpseudocode}
\usepackage{mathtools}
\usepackage{amsmath,color}
\usepackage{soul}
\newtheorem{theorem}{Theorem}
\newtheorem{definition}{Definition}
\newenvironment{proof}{\par\noindent{\bf Proof\ }}{\hfill\BlackBox\\[2mm]}

\DeclareMathOperator*{\argmin}{argmin}

\algnewcommand{\IIf}[1]{\State\algorithmicif\ #1\ \algorithmicthen}

\title{Suspicion-Free Adversarial Attacks on Clustering Algorithms}
\author{Anshuman Chhabra\thanks{Equal contribution and the names are in alphabetical order of second names.}\footnotemark[2], Abhishek Roy\footnotemark[1]\footnotemark[3], Prasant Mohapatra\footnotemark[2]\\
\footnotemark[2] Department of Computer Science, University of California, Davis\\
\footnotemark[3] Department of Electrical and Computer Engineering, University of California, Davis\\
\{chhabra, abroy, pmohapatra\}@ucdavis.edu
}
 
\begin{document}
 

\maketitle

\begin{abstract}
Clustering algorithms are used in a large number of applications and play an important role in modern machine learning-- yet, adversarial attacks on clustering algorithms seem to be broadly overlooked unlike supervised learning. In this paper, we seek to bridge this gap by proposing a black-box adversarial attack for clustering models for linearly separable clusters. Our attack works by perturbing a single sample close to the decision boundary, which leads to the misclustering of multiple unperturbed samples, named \textit{spill-over adversarial samples}. We theoretically show the existence of such adversarial samples for the K-Means clustering. Our attack is especially strong as (1) we ensure the perturbed sample is not an outlier, hence not detectable, and (2) the exact metric used for clustering is not known to the attacker. We theoretically justify that the attack can indeed be successful without the knowledge of the true metric. We conclude by providing empirical results on a number of datasets, and clustering algorithms. To the best of our knowledge, this is the first work that generates spill-over adversarial samples without the knowledge of the true metric ensuring that the perturbed sample is not an outlier, and theoretically proves the above.
\end{abstract}

\section{Introduction}
While machine learning (ML) and deep learning (DL) algorithms have been extremely successful at a number of learning tasks, \cite{papernot2016transferability} \cite{szegedy2013intriguing} showed that supervised classification algorithms can be easily fooled by adversarially generated samples. An adversarial sample is generally indistinguishable from the original sample by a human observer. However, it is still misclassifed by the classification algorithm. Since then, there has been a lot of research undertaken to make supervised models resilient to adversarial samples, and to expose vulnerabilities in existing ML/DL algorithms \cite{Carlini:2017:AEE:3128572.3140444}.

However, unsupervised learning algorithms, in particular \emph{clustering} algorithms, have seen little to no work that analyzes them from the adversarial attack perspective. This is in stark contrast to the significant role these algorithms play in modern data science, malware detection, and computer security \cite{malware2} \cite{malware1}. In a lot of cases, labels for data are \textit{hard} or even \textit{impossible} to obtain, rendering classification algorithms unfit for a large variety of learning problems. 

Designing adversarial examples for clustering algorithms is challenging because most clustering algorithms are inherently \emph{ad-hoc} unlike their supervised learning counterparts. Even defining what constitutes an adversarial sample is not trivial since the labels are absent. It is even more challenging to design a \emph{black-box} attack where the adversary has no knowledge of the clustering algorithm used. 

In this paper, we seek to answer these questions by first defining an adversarial sample for clustering, and then presenting a powerful \emph{black-box} adversarial attack algorithm against clustering algorithms for linearly separable clusters. Our attack algorithm is especially powerful because it generates an adversarially perturbed sample while ensuring that it is not detected as an outlier, and this leads to \emph{spill-over} \textit{adversarial} points. These are \textit{unperturbed} samples that are misclustered. Intuitively, the adversarial sample changes the decision boundary so that some other points get misclustered. Note it is important that the perturbed sample is not an outlier, otherwise the defender may simply discard it. In this way, our algorithm helps the attacker generate adversarial samples without arousing the suspicion of the defender.

Spill-over adversarial attacks on clustering have a number of motivating cases in the real world. Similar to \cite{crussell2015attacking}, where the authors present an attack algorithm against DBSCAN clustering \cite{sander1998density}, consider the AnDarwin tool \cite{crussell2013andarwin} that clusters Android apps into plagiarized and non-plagiarized clusters. If an adversary perturbs one plagiarized app such that it along with some other unperturbed apps, gets misclustered, that could lead to a loss of confidence in the tool as the defender is unaware of the reason for this decreased performance.

In summary, we make the following contributions:
\begin{itemize}
    \item We give a concise definition for the adversary's objective in the clustering setting (Section 3).
    \item We propose a black-box attack algorithm (Section 3) that perturbs a single sample while ensuring that it is not an outlier, which can then lead to additional samples being misclustered without adding any perturbation at all to those samples. To the best of our knowledge, this is the first work in the clustering domain where additional adversarial samples can be generated without any added noise to the corresponding data points. We name these samples \emph{spill-over adversarial} samples.
    \item We theoretically show (Section 4) existence of \emph{spill-over adversarial} samples for K-Means clustering \cite{Lloyd:2006:LSQ:2263356.2269955}. To the best of our knowledge this is the first theoretical result showing the existence of spill-over adversarial samples.
    \item We show that spill-over can happen even if the attacker does not know the exact metric used for clustering (Section 4). Our attack is stronger as it works with a noisy version of the true metric. In Section 5 we show that the attack is successful for different datasets where the true metric is not known. 
    \item We test our algorithm (Section 5) on Ward's Hierarchical clustering \cite{ward1963hierarchical}, and the K-Means clustering \cite{Lloyd:2006:LSQ:2263356.2269955} on multiple datasets, e.g., the UCI Handwritten Digits dataset \cite{uci_digits}, the MNIST dataset \cite{lecun1998mnist}, the MoCap Hand Postures dataset \cite{gardner2014measuring}, and the UCI Wheat Seeds dataset \cite{charytanowicz2010complete}. We find that our attack algorithm generates multiple spill-over adversarial samples across all datasets and algorithms. 
\end{itemize}

The rest of the paper is structured as follows: Section 2 discusses related work, Section 3 presents the threat model, and the proposed attack algorithm, Section 4 details the theoretical results, Section 5 presents the results, and Section 6 concludes the paper and discusses the scope for future work.

\section{Related Work}
The first works discussing clustering in an adversarial setting were \cite{skillicorn2009adversarial} and \cite{dutrisac2008hiding} where the authors discussed adversarial attacks that could lead to eventual misclustering using \emph{fringe} clusters where adversaries could place adversarial data points very close to the decision boundary of the original data cluster. \cite{biggio2013data} considered the adversarial attack on clusterings in a more detailed manner where they described the \emph{obfuscation} and \emph{poisoning} attack settings, and then provided results on single-linkage hierarchical clustering. In \cite{biggio2014poisoning} the authors extended their previous work to complete-linkage hierarchical clustering. \cite{crussell2015attacking} proposed a poisoning attack specifically for DBSCAN clustering, and also gave a defense strategy. 

As can be seen, very minimal research has discussed adversarial attacks on clustering. Moreover, the existing work has focused mainly on attacks for specific clusterings, instead of generalized black-box attacks as in this paper. While \cite{biggio2014poisoning} and \cite{biggio2013data} define obfuscation and poisoning attacks on clustering, none of these fit the bill in the perspective of our attack algorithm which generates spill-over adversarial samples. So we provide a simple yet effective definition for adversarial attacks on clustering in the next section. In the aforementioned related works the results are not theoretically justified as well. In this paper we prove the existence of spill-over adversarial samples for the K-Means clustering in Section 4. We also show, unlike \cite{biggio2013data}, that even if the adversary does not have the exact knowledge of clustering metric, spill-over adversarial samples can still exist. As discussed above, our attack is stronger as the adversarial sample generated is guaranteed to not be an outlier. This is not addressed in any of the previous works. We also present results on a number of different datasets for both the K-Means and Ward's clustering algorithms in Section 5.

\section{Proposed Attack Framework}

\subsection{Threat Model}

We first define the threat model and the role of the adversary. The main features of the threat model are as follows:  
\begin{enumerate}
    \item The adversary has no knowledge of the clustering algorithm that has been used and is thus, going to carry out a \emph{black-box} attack. 
    \item While the adversary does not have access to the algorithm, we assume that she has access to the datasets, and a noisy version of the true metric used for clustering. This assumption is weak as the metric can be learnt by observing the clustering outputs by the defender. For details see \cite{xing2003distance} and the references therein. We therefore first present our algorithm assuming exact knowledge of metric. We then show in Section 4 that, if the noise is small enough, under certain conditions, a spill-over adversarial sample in clustering using the noisy metric will also spill-over in clustering using the true metric.  
    \item Once the attacker has the clusters, she can use the adversarial attack algorithm provided in the subsequent subsection to perturb just one judiciously chosen input sample, called the \textit{target sample}. The algorithm perturbs this input data sample by iteratively crafting a precise additive noise to generate the spill-over adversarial samples. We allow the perturbation to be different for each feature of the sample. We assume that each perturbation is within a pre-specified threshold which is determined by adversary's motivation of not getting detected as an outlier, and/or the limited attack budget of the adversary.
\end{enumerate}

\subsubsection{The Adversary's Objective}
Before stating the adversary's objective, we need to define misclustering as it is not well-defined unlike supervised learning. The goal of the adversary is to maximize the number of points which are misclustered into the target cluster. Note that, perturbation of the target sample essentially perturbs the decision boundary which creates spill-over adversarial samples but the target sample may not necessarily be misclustered. The adversarial attack algorithm presented in this paper operates on only two clusters at a time. This is similar to a targeted adversarial attack on supervised learning models \cite{Carlini:2017:AEE:3128572.3140444}.

\begin{table}[t]
\caption{Parameters and Notation}
\begin{center}
\begin{tabular}{ |p{1.3cm}|p{6.35cm}| }
\hline
Notation & Meaning\\
\hline
\hline
 $X$ & Dataset used for clustering, $X \in \mathbb{R}^{n \times m}$\\
 \hline
 $n$ & Number of samples in $X$\\
 \hline
 $m$ & Number of features in $X$\\
 \hline
 $C$ & The clustering algorithm used\\
 \hline
 $k_i$ & $i^{th}$ cluster, $i=1,2$\\
 \hline
 $n_i$ & Number of samples in $i^{th}$ cluster, $i=1,2$\\
 \hline
 $X_{k_{i}}$ & Set of data samples of $X$ in $k_i$, $X_{k_{i}} \in \mathbb{R}^{n_i \times m}$, $i=1,2$\\
 \hline
 $Y$ & Clustering result as $n \times 2$ matrix, $Y \in \{0,1\}^{n \times 2}$\\
 \hline
 $x_t$ & Target sample to perturb in $X$\\
 \hline
 $\Delta$ & Acceptable noise threshold for each feature, $\Delta \in \mathbb{R}^m$\\
 \hline
 $\epsilon^*$ & Optimal additive perturbation, $\epsilon^* \in \mathbb{R}^m$\\
 \hline
 $\delta$ & Defined metric to measure cluster change\\
 \hline
 $x_t'$ & Target sample after perturbation\\
 \hline
 $X'$ & $X$ with $x_t$ replaced by $x_t'$\\
 \hline
 $Y'$ & Clustering result after noise addition as $n \times 2$ matrix, $Y' \in \{0,1\}^{n \times 2}$\\
 \hline
  $S$ & Set of spill-over adversarial samples\\
 \hline
 $n_s$ & Number of samples in $S$\\
 \hline
  $c_i$ & Centroid of cluster $k_i$, $i=1,2$\\
 \hline
 $\mathcal{D}_C$ & Mahalanobis Depth for clusters \\
 \hline
 $\mathcal{D}_M$ & Coordinatewise-Min-Mahalanobis Depth for high dimensional clusters \\
 \hline
\end{tabular}
\end{center}
\end{table}
\subsection{Problem Formulation}
We list all the notation and parameters used in the rest of the paper in Table 1. Also, the Frobenius norm of a matrix $M$ is denoted by  $\|A\|_{F} = (\sum_{i,j} A_{ij}^2)^{1/2}$, and $\langle \cdot, \cdot \rangle$ represents the standard vector inner product. Let a clustering algorithm be defined as a function $C:\mathbb{R}^{n \times m}\rightarrow \{0,1\}^{n \times k}$ that partitions the given dataset $X \in \mathbb{R}^{n \times m}$ into $k$ clusters. Since we are only considering hard clustering problems, we can represent the clustering result as a matrix $Y \in \{0,1\}^{n \times k}$ where each row has all 0 except one 1 indicating the cluster that point belongs to, and thus, $Y = C(X)$. Here $n$ refers to the number of samples in the dataset and $m$ refers to the features. The adversary can only perturb a single sample $x_t$, the target sample, which lies in the dataset, that is $X = \{x_j\}_{j=1}^n$. The procedure for the selection of the target point $x_t$ by the adversary is explained in the next section when the proposed attack algorithm is discussed. 

In this paper we consider 2-way clustering, i.e., $k = 2$, with the two clusters denoted as $k_1$ and $k_2$. The number of samples in each cluster are denoted as $n_1$ and $n_2$, respectively. Now, assuming $x_t$ belongs to cluster $k_1$, that is $Y_{x_{t}, k_{1}} = 1$, the aim of the adversary then is to perturb $x_t$ to $x_{t}'$ in such a way that a subset of data points $S \subseteq X_{k_1}$ change their cluster membership from $k_1$ to $k_2$. The new dataset containing $x_{t}'$ instead of $x_t$ is denoted as $X'$. Thus, in the resulting clustering output $Y' = C(X')$, for all $x_i \in S$, the attack leads to $Y_{x_{i},k_{2}}' = 1$ (and $Y_{x_{i},k_{1}}' = 0$) whereas in the original clustering, $Y_{x_{i},k_{1}} = 1$ (and $Y_{x_{i},k_{2}}' = 0$). The set $S$ is what we call the set of \emph{spill-over} adversarial samples. 
\subsection{Proposed Black-box Attack on Clustering}
\begin{algorithm}
\caption{Proposed Black-box Adversarial Attack}\label{alg:attack}
 \textbf{Input:} $X$, $C$, $\Delta$, $Y = C(X)$, $k_1$, $k_2$, $n_1$, $n_2$, $X_{k_{1}}$,$X_{k_{2}}$\\
 \textbf{Output:} Optimal additive perturbation $\epsilon^* \in \mathbb{R}^m$ \\
 \begin{algorithmic}[1]
 \State \textbf{set} $c_2 \leftarrow \frac{1}{n_2} \smashoperator{\sum\limits_{x_j \in X_{k_{2}}}} x_j$
 \State \textbf{set} $x_t \leftarrow \argmin_{x \in X_{k_{1}}} |x - c_2|$
 \Function{$f$}{$\epsilon$}
 \State \textbf{set} $x_{t}' \leftarrow x_t + \epsilon$
 \State \textbf{obtain} $X'$ \textbf{from} $X$ \textbf{by replacing} $x_t$ \textbf{with} $x_t'$  
 \State \textbf{obtain} $Y' = C(X')$
 \State \textbf{set} $\delta \coloneqq - \|YY^T-Y'Y'^T\|_{F}$
 \State \textbf{return} $\delta$
\EndFunction
\State \textbf{minimize} $f(\epsilon^*)$ \textbf{subject to} $\epsilon^{*}_j \in [-\Delta_j, \Delta_j]$ \textbf{where} $j = 1,2,...,m$ 
\State \textbf{return} $\epsilon^*$
 \end{algorithmic}
\end{algorithm}
The proposed algorithm is shown as Algorithm 1. The inputs for the algorithm are the dataset $X \in \mathbb{R}^{n \times m}$, the clustering algorithm $C$, the clustering result on the original data $Y \in \{0,1\}^{n \times 2}$, the data points that populate each of the two clusters, $X_{k_{1}} \in \mathbb{R}^{n_1 \times m}$ and $X_{k_{2}} \in \mathbb{R}^{n_2 \times m}$, and the noise threshold $\Delta \in \mathbb{R}^m$ where $k_i$ ($i=\{1,2\}$) denotes the clusters. $\Delta$ is the noise threshold for each of the $m$ features, i.e., the $j^{th}$ feature of the optimal perturbation will lie in the range $[-\Delta_j, \Delta_j]$ where $j = 1,..,m$. This definition for $\Delta$ can lead to the case where points of $k_2$ spill-over into $k_1$, but since the formulation is equivalent, we consider only the case of spill-over from $k_1$ to $k_2$ in the paper. $\Delta$ ensures that the adversary does not perturb the target sample too much to get detected by the defender as an outlier. $\Delta$ can also be interpreted as the limited attack budget of the  adversary. We elaborate on how to choose $\Delta$ at the end of this section. \\ 
Algorithm 1 proceeds as follows: In Line 1, we find the centroid of whichever cluster we want the spill-over points to be a part of, after the attack. From here on, without loss of generality, we assume the spill-over points belong to cluster $k_1$ originally, and therefore cluster centroid $c_2$ for $k_2$ is calculated. Next, in Line 2, we select the target point $x_t$ in $k_1$ which is closest in Euclidean distance to $c_2$. This point is a good target for the adversarial attack as it is the nearest point of $k_1$ to the decision boundary between both clusters. Lines 3-10 define the function $f(\epsilon) \in \mathbb{R}$ which we optimize over to find the $\epsilon$ that will lead to spill-over.

In Line 4 of the algorithm, we perturb the target point and obtain $x_t'$, and get $X'$ by replacing $x_t$ with $x_t'$ in $X$. We then find the noisy clustering result $Y' = C(X')$. Line 5 presents the metric $\delta$ used to measure how much the clustering result has \textit{changed} from the original clustering to after the attack \cite{biggio2013data}:
\begin{equation}
    \delta \coloneqq - \|YY^T-Y'Y'^T\|_{F} \label{eq:defdel}
\end{equation}

The $ij^{th}$ element of the $YY^T$ matrix represents whether sample $i$, and $j$ belong to the same cluster. Note that, if there is no change in cluster membership, $\delta = 0$. $|\delta|$ increases with the number of points that spill over from $k_1$ to $k_2$.

Line 11 is essentially the formulation of the minimization problem. We have to find the optimal perturbation $\epsilon_j^* \in [-\Delta_j, \Delta_j]$ which minimizes $f$, such that $f(\epsilon^*) \leq f(\epsilon)$ for any $\epsilon \in [-\Delta_j, \Delta_j]$, $j = 1,2,\cdots,m$. It is also important to understand a few aspects about the function $f$, before we get to the choice of an optimization approach. As the function is not continuous, we cannot use gradient based methods to solve the minimization problem. Instead, we require derivative-free black-box optimization approaches to minimize $f$ while ensuring that the noise threshold constraints on the optimal perturbation $\epsilon^*$ are met. There are many possibilities for such an optimization procedure, e.g., genetic algorithms \cite{goldberg2006genetic}, and simulated annealing \cite{kirkpatrick1983optimization}. We opt for a cubic radial basis function (RBF) based surface response methodology \cite{knysh2016blackbox} for the optimization. The optimization approach utilizes a modified iterative version of the CORS algorithm \cite{Regis2005}, and uses the Latin hypercube approach \cite{mckay2000comparison} for the initial uniform sampling of search points required for the CORS algorithm. The CORS optimization procedure has achieved competitive results on a number of different test functions \cite{holmstrom2008adaptive}, \cite{Regis2005}. For our attack algorithm, this optimization algorithm achieved much better results on multiple datasets as compared to methods like genetic algorithm, and simulated annealing. We found that this optimization was much less sensitive to parameter choices. We present the results obtained in Section 5.
\subsubsection{Choosing $\Delta$}
If the adversary does not have an attack budget, then $\Delta$ should only be chosen such that the adversarial sample does not get construed as an outlier. Mahalanobis Depth (MD) is one such measure for \textit{outlyingness}:
\begin{definition}[Mahalanobis Depth]
Mahalanobis Depth of a point $x$, $\mathcal{D}\left(x\right)$, with respect to a set $X\subseteq \mathbb{R}^m$ is defined as \begin{align}\label{eq:defmd}
    \mathcal{D}\left(x\right)=\left(1+\left(x-\bar{x}\right)^\top\hat{\Sigma_x}^{-1}\left(x-\bar{x}\right)\right)^{-1}
\end{align}
where $\bar{x}$, and $\hat{\Sigma_x}$ are the sample mean and covariance. 
\end{definition}
The smaller the value of $\mathcal{D}$, the larger is \textit{outlyingness}. Using $\mathcal{D}$ to detect the \textit{outlyingness} of a point for a dataset with clusters may pose problems, e.g., for two well separated clusters the points around the line joining the cluster means have very small depth. But a point between two clusters which is sufficiently far from both the clusters will clearly be interpreted as an outlier. So we propose a modified measure of depth similar in flavor to \cite{paindaveine2013depth}:
\begin{definition}[Mahalanobis Depth for Clusters (MDC)]
Let there be $J$ clusters. Say $\mathcal{D}\left(x\right)$ with respect to only cluster $i$ is given by $t_i$. Then Mahalanobis Depth for clusters of $x$ is defined as $\mathcal{D}_C\left(x\right)=\sum_{i=1}^{J}t_i$.
\end{definition}
For high dimensional data computing MD is difficult as the computed covariance matrix can be singular. So for high dimensional data we propose a new depth-based measure, Coordinate-wise Min-Mahalanobis-Depth (COMD), to measure outlyingness:

\begin{definition}[Coordinate-wise Min-Mahalanobis-Depth]
Consider $x=\left[x_1,x_2,\cdots,x_m\right]\in X\subseteq \mathbb{R}^{n\times m}$. Let $X_i$ denote the $i^{th}$ column of $X$. Let $\mathcal{D}_{C,i}$ denote the MDC depth of $x_i$ w.r.t the points $X_i\subseteq \mathbb{R}^{n\times 1}$. Then Coordinate-wise Min-Mahalanobis-Depth is defined as $\mathcal{D}_M\left(x\right)=\min\{\mathcal{D}_{C,i}\}_{i=1}^m$.  
\end{definition}

Intuitively COMD measures the maximum \textit{outlyingness} along all the coordinates. It is a conservative measure of outlyingness, and hence ensuring small value of COMD is sufficient for the adversary to avoid being detected as an outlier. After observing the clusters, the attacker forms equi-$\mathcal{D}_C\left(x\right)$ contours, or equi-$\mathcal{D}_M\left(x\right)$ spaces for $m\geq 2$, over the data as we will show in the toy example in Section 4.1. $\Delta$ is chosen such that the perturbed point is at least above 0.1 quantile of the COMD values of the dataset. 

\begin{figure*} 
\centering
\begin{subfigure}{.25\textwidth}
  \centering
  \includegraphics[scale=0.34]{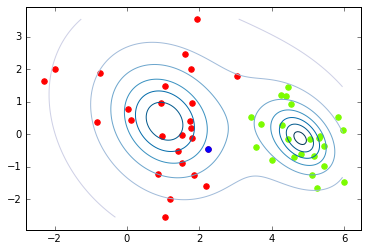}%
  \caption{}
\end{subfigure}%
\begin{subfigure}{.25\textwidth}
  \centering
  \includegraphics[scale=0.34]{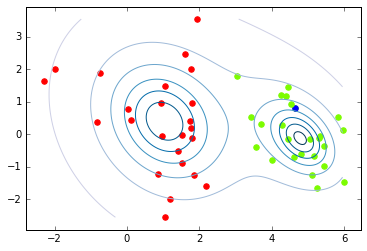}%
  \caption{}
\end{subfigure}%
\begin{subfigure}{.25\textwidth}
  \centering
  \includegraphics[scale=0.34]{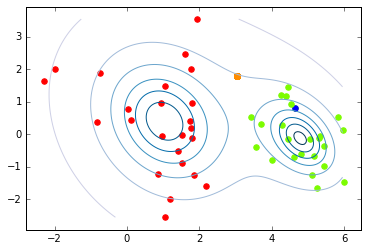}%
  \caption{}
\end{subfigure}%
\begin{subfigure}{.25\textwidth}
  \centering
  \includegraphics[scale=0.34]{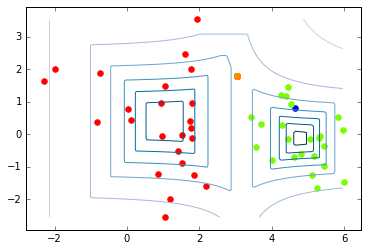}%
  \caption{}
\end{subfigure}
\caption{Results on toy data after using Algorithm 1 for spill-over attack} \label{fig:toy}
\end{figure*}

\section{Theoretical Results}
In this section we theoretically study the effect of perturbation, and using a noisy metric. The following theorem shows that perturbing one sample can distort the decision boundary in such a way that another uncorrupted point can spill-over. 
\begin{theorem}
Say k-means clustering is used to cluster a linearly separable dataset. A judiciously chosen datapoint can be perturbed by additive noise, ensuring that it does not become an outlier, in such a way that there may exist another point which changes cluster membership, i.e., one or more spill-over point(s) may exist.
\end{theorem}
\begin{proof}
Let a point $x$ is such that $\langle x-c_1,c_2-c_1\rangle\geq 0$. Intuitively we select $x$ belonging to $k_1$ such that $x-c_1$ forms an acute angle with $c_2-c_1$. It is easy to see that such a point always exists. Now, an adversary perturbs $x$ to $c_2$, hence ensuring that the perturbed point is not an outlier. We will show by contradiction that, under certain conditions, there exists at least one other point which will spill over to $k_2$.\\
Let us assume that $k_1$ remain unchanged after the perturbation except that $x$ moves to $c_2$. Because of this perturbation the mean and composition of the $k_2$ does not change. Let the mean of the $k_1$ be $c_1'$ after the perturbation. We have,\\
$c_1-c_1'=c_1-\frac{n_1c_1-x}{n_1-1}=\frac{x-c_1}{n_1-1}$.
Say there is a point $y$ such that $\langle x-c_1,y-c_1\rangle\geq 0$, and $\|y-c_2\|^2=\|y-c_1\|^2+\alpha$ with $\alpha\geq 0$ . Consequently,  $\langle y-c_1,c_1-c_1'\rangle=\langle y-c_1,\frac{x-c_1}{n_1-1}\rangle\geq 0$. Now we have,
\begin{align*}
    &\|y-c_1'\|^2\\
    =&\|y-c_1+c_1-c_1'\|^2\\
    =& \|y-c_2\|^2+\|c_1-c_1'\|^2+2\langle y-c_1,c_1-c_1'\rangle -\alpha
\end{align*}
The second and third term in the above expression is non-negative. If $\alpha\leq \|c_1-c_1'\|^2+2\langle y-c_1,c_1-c_1'\rangle$, then the point $y$ is closer to $c_2$ and should be in $k_2$. Contradiction!
\end{proof}
\vspace{1mm}
\indent \indent As discussed above, we assume that the adversary has access to a noisy version of the true metric. In the following theorem we show that a spill-over adversarial sample under noisy metric will spill-over under clustering with true metric under certain conditions.  
\begin{theorem}
Let there be a point $y$ which spilled over from $k_1$ to $k_2$ of the dataset $X$ due to a attack following Algorithm~\ref{alg:attack} with distance metric $d'\colon X \times X \rightarrow \mathbb{R}^+$. If the true distance metric $d\colon X \times X \rightarrow \mathbb{R}^+$ used for clustering by the defender satisfies 
\begin{align} \label{eps_perturbation}
    \max\{0,d(u,v) - \zeta\} \leq d'(u,v) \leq d(u,v) + \zeta
\end{align}
$\forall (u,v) \in X$, and  $\zeta \geq 0$, then, under certain conditions, $y$ will spill-over in the clustering using metric $d$ as well. 
\end{theorem}
\begin{proof}
Let the cluster centers for $d\left(d'\right)$ be $c_{1}\left(c_{1}'\right)$ and $c_{2}\left(c_{2}'\right)$, and the clusters be $k_{1}\left(k_{1}'\right)$, and $k_{2}\left(k_{2}'\right)$. We assume that the attacker can query to find out which point of $k_1'$ is closest to $c_2$. We use the attack used to prove Theorem 1, i.e., the attacker perturbs a point $x$ from $k_1'$, to the center of $k_2'$. Consequently, $c_{1}'$ becomes $\bar{c_{1}}'$, $c_{2}'$ remains the same, and another point $y$ spills-over from $k_1'$ to $k_2'$. 

Say, after the attack, in the actual clustering using $d$, the centers are represented by $\bar{c_1}$ and $\bar{c_2}$. 
We will show that $y$ will spill-over in the actual clustering using $d$ as well, under certain conditions.
Now, we have before the attack:
\begin{equation} \label{y_before_attack}
    d'(y,c_{2}') > d'(y,c_{1}')  
\end{equation}
Using the triangle inequality twice on (\ref{y_before_attack}) we can write:
\begin{align*}
     d'(y,c_{2}') & > d'(y,\bar{c_{1}}') - d'(\bar{c_{1}}', c_{1}')\\
     d'(y,\bar{c_2}) + d'(\bar{c_2}, c_{2}') & > d'(y,\bar{c_{1}}) - d'(\bar{c_{1}}', \bar{c_1}) - d'(\bar{c_{1}}', c_{1}')
\end{align*}
Let $\gamma = d'(\bar{c_{1}}', c_{1}') + d'(\bar{c_{1}}', \bar{c_1}) + d'(\bar{c_2}, c_{2}') \geq 0 $, and using (\ref{eps_perturbation}) we can write:
\begin{align}\label{gamma_eq}
    d'(y,\bar{c_2}) & - d'(y,\bar{c_{1}})  > - \gamma \nonumber\\
    d(y,\bar{c_2}) & - d(y,\bar{c_{1}})  > - \gamma - 2\zeta
\end{align}

If $d(y,\bar{c_2}) - d(y,\bar{c_1}) < 0$ then point $y$ has spilled-over in the actual clustering too, and we can see that the lower bound in (\ref{gamma_eq}) is negative as both $\zeta$ and $\gamma$ are non-negative. Therefore, this ensures that $d(y,\bar{c_2}) - d(y,\bar{c_1})$ is negative for a range of values of $y$.
\end{proof}

\subsection{Toy Example}

In this subsection we present the working of our attack algorithm consistent with the assumptions in Theorem 1. We create 2-dimensional Gaussian clusters with standard deviations of $1.45$ and $0.75$, and cluster centroids at $(1,0)$ and $(5,0)$, respectively. Using Algorithm 1 we find the target point $x_t$ to perturb which is originally in $k_1$. The clusters $k_1$ and $k_2$ generated along with $x_t$ are shown in Fig. 1(a). The first cluster $k_1$, the second cluster $k_2$, and the target sample $x_t$ are shown in red, green, and blue respectively. 

It is important to note that the assumption taken in Theorem 1 regarding the target sample also holds true as $\langle x_t -c_1,c_2-c_1\rangle = 5.0511$, where $c_1$ and $c_2$ denote the cluster centroids of $k_1$ and $k_2$. Next, using the optimization procedure outlined in Algorithm 1, we perturb $x_t$ in such a way so as to lead to spill-over. Figure 1(b) shows that the perturbed $x_t$ has changed cluster membership, and there is one spill-over adversarial sample. Figure~\ref{fig:toy} also shows the equi-$\mathcal{D}_C$ contours with depth decreasing by 0.1, away from the cluster centers. The contours in Fig.~\ref{fig:toy} show that the adversarial sample is not an outlier. Figure 1(a)-(c) show the equi-$\mathcal{D}_C$ contours, and Fig. 1(d) shows the equi-$\mathcal{D}_M$ contours. Note that $\mathcal{D}_C$ correctly prescribes more \textit{outlyingness} for points that lie between two clusters, which do not belong to either cluster, as compared to points which belong to one of the clusters. The spill-over adversarial sample has been highlighted in orange in Fig. 1(c). We find that the spill-over adversarial sample $y$ satisfies the condition stated in Theorem 1, i.e., $\langle x_t-c_1,y-c_1\rangle = 1.2872\geq 0$.

\section{Results}
In this section, we present the results obtained when we use Algorithm 1 to generate spill-over adversarial samples for Ward's Hierarchical clustering and K-Means clustering on a number of different datasets. The cubic RBF surface response method used for the gradient-free optimization of our objective function in Algorithm 1 is based off the open-source implementation of \cite{knysh2016blackbox} that is available on GitHub \cite{blackbox} which we have slightly modified programmatically so that it can handle high dimensional data for optimization. We have also open-sourced the code used to generate all results using our proposed attack algorithm on GitHub \cite{code}. We use the K-means clustering and Ward's clustering implementation as available in the Scikit-learn package \cite{pedregosa2011scikit}.
\subsection{Ward's Hierarchical Clustering}
\subsubsection{UCI Handwritten Digits Dataset}
The UCI Digits dataset \cite{uci_digits} consists of $8\times 8$ images of handwritten digits from $0$ to $9$. In these images each pixel is an integer between 0, and 16. Each image can be represented by feature vectors of length $64$. We test Ward's clustering for this dataset since it clusters the digits well. We use these images as inputs to the clustering algorithm.
\begin{table*}[t]
\caption{UCI Handwritten Digits Dataset results (Ward's Clustering)}
\begin{center}
\begin{tabular}{ |c||c|c|c|c|c|c|c| }
\hline
Digit clusters & $k_1$ & $k_2$ & $n_1$ & $n_2$ & \# Misclustered samples & $\mathcal{D}_M(x_t')$ & $\mathcal{D}_M(X)$ quantile\\
\hline
\hline
 Digit \texttt{1} \& \texttt{4} & Digit \texttt{4} & Digit \texttt{1} & 181 & 182 & 24 & 0.061 & 0.10\\
 \hline
 \hline
 Digit \texttt{8} \& \texttt{9} & Digit \texttt{9} & Digit \texttt{8} & 164 & 190 & 21 & 0.114 & 0.285\\
 \hline
\end{tabular}
\end{center}
\end{table*}

We apply Ward's clustering on two clustering problems: For clustering Digits \texttt{1} and \texttt{4} images, and clustering Digits \texttt{8} and \texttt{9} images. For each case, the cluster information, parameter details, total misclustered samples, $\mathcal{D}_M\left(x_t'\right)$, and the quantile of $\mathcal{D}_M(x_t')$ with respect to $\mathcal{D}_M(X)$, are listed in Table 2. Algorithm 1 starts with the target sample from $k_1$, and generates spill-over adversarial samples that switch cluster membership from $k_1$ to $k_2$. For the Digits \texttt{1} and \texttt{4} clusters the spill-over adversarial images are shown in Fig. 2, and for Digits \texttt{8} and \texttt{9}, they are shown in Fig. 3. For both the clustering problems, the $\mathcal{D}_M(X)$ quantile above which $\mathcal{D}_M(x_t')$ lies, indicates that it cannot be an outlier.

\begin{figure}[t]
\centering
\begin{subfigure}{.075\textwidth}
  \centering
  \includegraphics[scale=0.10]{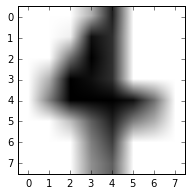}%
  \hfill
  \includegraphics[scale=0.10]{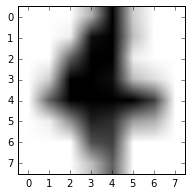}%
  \hfill
  \includegraphics[scale=0.10]{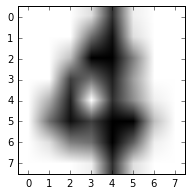}%
  \hfill
  \includegraphics[scale=0.10]{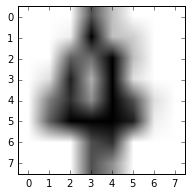}
\end{subfigure}%
\begin{subfigure}{.075\textwidth}
  \centering
  \includegraphics[scale=0.10]{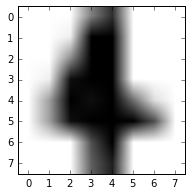}%
    \hfill
  \includegraphics[scale=0.10]{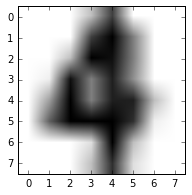}%
  \hfill
  \includegraphics[scale=0.10]{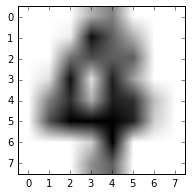}%
  \hfill
  \includegraphics[scale=0.10]{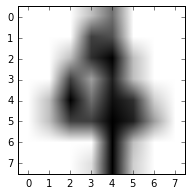}
\end{subfigure}%
\begin{subfigure}{.075\textwidth}
  \centering
  \includegraphics[scale=0.10]{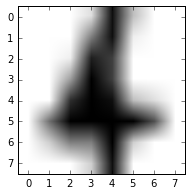}%
    \hfill
  \includegraphics[scale=0.10]{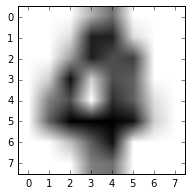}%
  \hfill
  \includegraphics[scale=0.10]{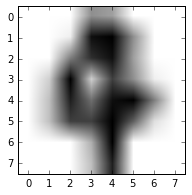}%
  \hfill
  \includegraphics[scale=0.10]{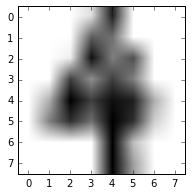}
\end{subfigure}%
\begin{subfigure}{.075\textwidth}
  \centering
  \includegraphics[scale=0.10]{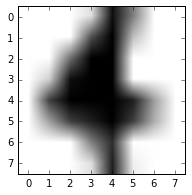}%
    \hfill
  \includegraphics[scale=0.10]{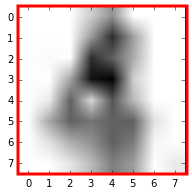}%
  \hfill
  \includegraphics[scale=0.10]{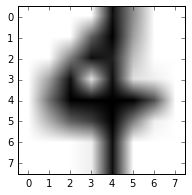}%
  \hfill
  \includegraphics[scale=0.10]{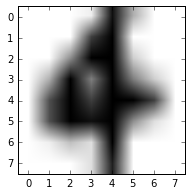}
\end{subfigure}%
\begin{subfigure}{.075\textwidth}
  \centering
  \includegraphics[scale=0.10]{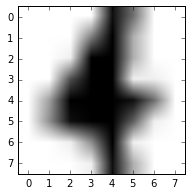}%
    \hfill
  \includegraphics[scale=0.10]{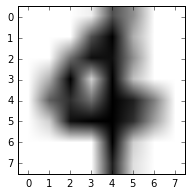}%
  \hfill
  \includegraphics[scale=0.10]{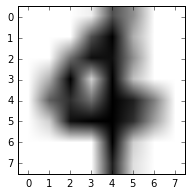}%
  \hfill
  \includegraphics[scale=0.10]{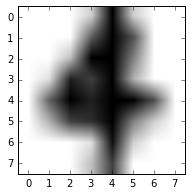}
\end{subfigure}%
\begin{subfigure}{.075\textwidth}
  \centering
  \includegraphics[scale=0.10]{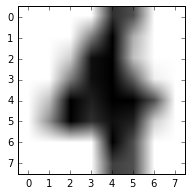}%
    \hfill
  \includegraphics[scale=0.10]{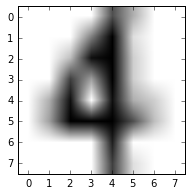}%
  \hfill
  \includegraphics[scale=0.10]{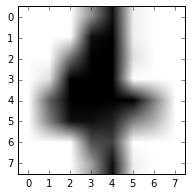}%
  \hfill
  \includegraphics[scale=0.10]{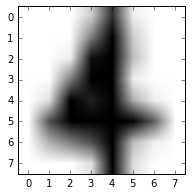}
\end{subfigure}
\caption{Misclustered images that switched clusters from the Digit \texttt{4} to the Digit \texttt{1} cluster (adversarially perturbed sample shown with red border)}
\end{figure}
\begin{figure}[t]
\centering
\begin{subfigure}{.065\textwidth}
  \centering
  \includegraphics[scale=0.10]{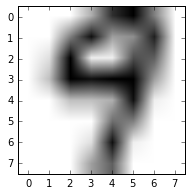}%
  \hfill
  \includegraphics[scale=0.10]{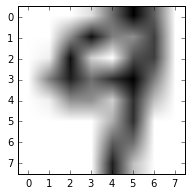}%
  \hfill
  \includegraphics[scale=0.10]{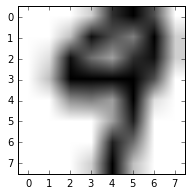}
\end{subfigure}%
\begin{subfigure}{.065\textwidth}
  \centering
  \includegraphics[scale=0.10]{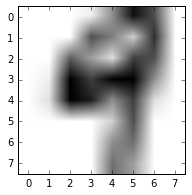}%
    \hfill
  \includegraphics[scale=0.10]{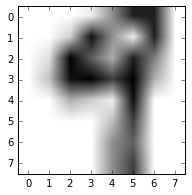}%
  \hfill
  \includegraphics[scale=0.10]{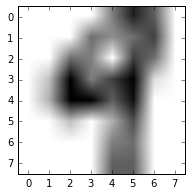}
\end{subfigure}%
\begin{subfigure}{.065\textwidth}
  \centering
  \includegraphics[scale=0.10]{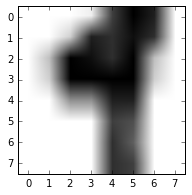}%
    \hfill
  \includegraphics[scale=0.10]{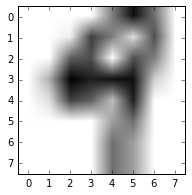}%
  \hfill
  \includegraphics[scale=0.10]{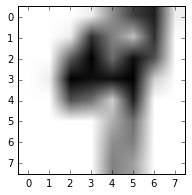}
\end{subfigure}%
\begin{subfigure}{.065\textwidth}
  \centering
  \includegraphics[scale=0.10]{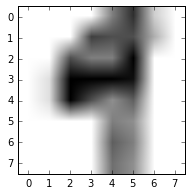}%
    \hfill
  \includegraphics[scale=0.10]{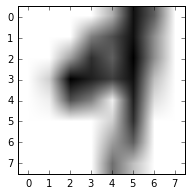}%
  \hfill
  \includegraphics[scale=0.10]{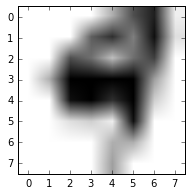}
\end{subfigure}%
\begin{subfigure}{.065\textwidth}
  \centering
  \includegraphics[scale=0.10]{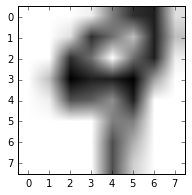}%
    \hfill
  \includegraphics[scale=0.10]{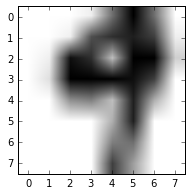}%
  \hfill
  \includegraphics[scale=0.10]{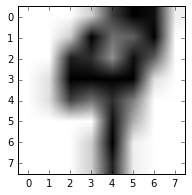}
\end{subfigure}%
\begin{subfigure}{.065\textwidth}
  \centering
  \includegraphics[scale=0.10]{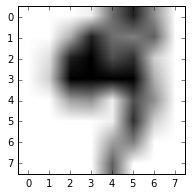}%
    \hfill
  \includegraphics[scale=0.10]{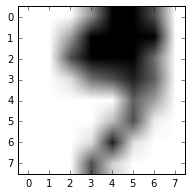}%
  \hfill
  \includegraphics[scale=0.10]{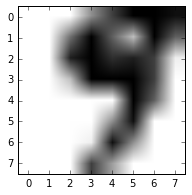}
\end{subfigure}%
\begin{subfigure}{.065\textwidth}
  \centering
  \includegraphics[scale=0.10]{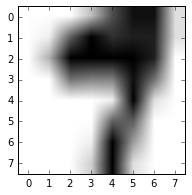}%
    \hfill
  \includegraphics[scale=0.10]{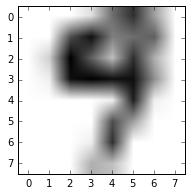}%
  \hfill
  \includegraphics[scale=0.10]{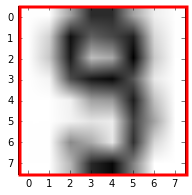}
\end{subfigure}
\caption{Misclustered images that switched clusters from the Digit \texttt{9} to the Digit \texttt{8} cluster (adversarially perturbed sample shown with red border)}
\end{figure}
\subsubsection{MNIST Dataset}
To show the performance of Algorithm 1 we use the  \texttt{MNIST} dataset \cite{lecun1998mnist} which is an image dataset. We utilize small subsets of the original digit images, and use 200 images for each digit. The digit images here are $28\times 28$ grayscale images of digits from $0$ to $9$. For inputs to the clustering, we flatten each image sample and get a feature vector of length $m=784$.
\begin{table*}[t]
\caption{\texttt{MNIST} Dataset results (Ward's Clustering)}
\begin{center}
\begin{tabular}{ |c||c|c|c|c|c|c|c| }
\hline
Digit clusters & $k_1$ & $k_2$ & $n_1$ & $n_2$ & \# Misclustered samples & $\mathcal{D}_M(x_t')$ & $\mathcal{D}_M(X)$ quantile\\
\hline
\hline
 Digit \texttt{1} \& \texttt{4} & Digit \texttt{4} & Digit \texttt{1} & 192 & 208 & 11 & 0.067 & 0.49\\
 \hline
 Digit \texttt{2} \& \texttt{3} & Digit \texttt{3} & Digit \texttt{2} & 176 & 224 & 2 & 0.13 & 0.828 \\ 
 \hline
\end{tabular}
\end{center}
\end{table*}
We apply Ward's clustering on two clustering problems: For clustering Digits \texttt{1} and \texttt{4} images, and clustering Digits \texttt{2} and \texttt{3} images. For each of these, the cluster information, parameter details, total misclustered samples, the perturbed target sample depth $\mathcal{D}_M(x_t')$, and the quantile of $\mathcal{D}_M(x_t')$ with respect to $\mathcal{D}_M(X)$, are listed in Table 3. The attack algorithm starts with the target sample from the $k_1$ cluster and then generates spill-over adversarial samples that switch cluster membership from $k_1$ to $k_2$. For the Digits \texttt{1} and \texttt{4} clusters the spill-over adversarial images are shown in Fig. 4, and for Digits \texttt{2} and \texttt{3} clusters they are shown in Fig. 5. Here too, the $\mathcal{D}_M(X)$ quantile range that $\mathcal{D}_M(x_t')$ lies in ensures that the perturbed adversarial sample cannot be detected as an outlier.
\begin{figure}[t]
\centering
\begin{subfigure}{.041\textwidth}
  \centering
  \includegraphics[scale=0.10]{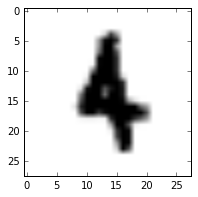}
\end{subfigure}%
\begin{subfigure}{.041\textwidth}
  \centering
  \includegraphics[scale=0.10]{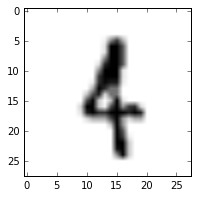}
\end{subfigure}%
\begin{subfigure}{.041\textwidth}
  \centering
  \includegraphics[scale=0.10]{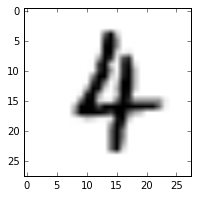}
\end{subfigure}%
\begin{subfigure}{.041\textwidth}
  \centering
  \includegraphics[scale=0.10]{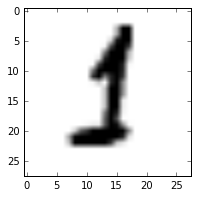}
\end{subfigure}%
\begin{subfigure}{.041\textwidth}
  \centering
  \includegraphics[scale=0.10]{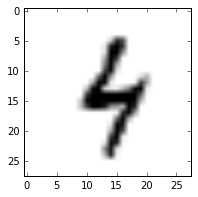}
\end{subfigure}%
\begin{subfigure}{.041\textwidth}
  \centering
  \includegraphics[scale=0.10]{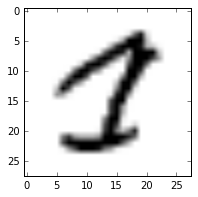}
\end{subfigure}%
\begin{subfigure}{.041\textwidth}
  \centering
  \includegraphics[scale=0.10]{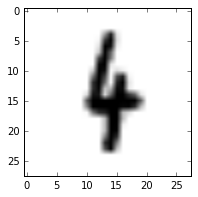}
\end{subfigure}%
\begin{subfigure}{.041\textwidth}
  \centering
  \includegraphics[scale=0.10]{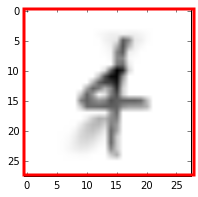}
\end{subfigure}%
\begin{subfigure}{.041\textwidth}
  \centering
  \includegraphics[scale=0.10]{mnist-1-4-2.png}
\end{subfigure}%
\begin{subfigure}{.041\textwidth}
  \centering
  \includegraphics[scale=0.10]{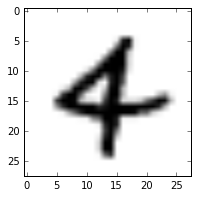}
\end{subfigure}%
\begin{subfigure}{.041\textwidth}
  \centering
  \includegraphics[scale=0.10]{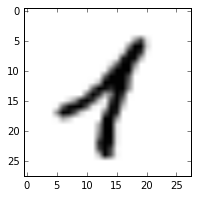}
\end{subfigure}
\caption{Misclustered \texttt{MNIST} images that switched clusters from the Digit \texttt{4} to the Digit \texttt{1} cluster (adversarially perturbed sample shown with red border)}
\end{figure}
\begin{figure}[t]
\centering
\begin{subfigure}{.1\textwidth}
  \centering
  \includegraphics[scale=0.10]{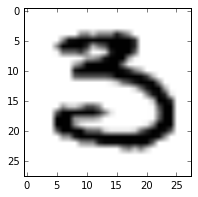}%
\end{subfigure}%
\begin{subfigure}{.1\textwidth}
  \centering
  \includegraphics[scale=0.10]{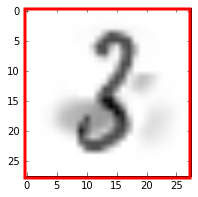}%
\end{subfigure}
\caption{Misclustered \texttt{MNIST} images that switched clusters from the Digit \texttt{3} to the Digit \texttt{2} cluster (adversarially perturbed sample shown with red border)}
\end{figure}
\subsection{K-Means Clustering}
\subsubsection{UCI Wheat Seeds Dataset}
The UCI Wheat Seeds dataset \cite{charytanowicz2010complete} contains measurements of geometric properties of three different varieties of wheat kernels: Kama, Rosa, and Canadian, with 70 samples for each seed variety. Each sample has the following 7 features: Area of the kernel $A$, perimeter of the kernel $P$, compactness $C = 4 \pi A/P^2$, kernel length, kernel width, asymmetry coefficient, and length of kernel groove. 
We use the K-Means clustering for clustering Kama and Rosa wheat kernels. Cluster sizes for Rosa is $n_1 = 79$, and for Kama is $n_2 = 61$. The noise threshold $\Delta$ is selected using the outlier depth methodology described in the previous sections. Algorithm 1 starts with the target sample in the Rosa cluster and generates 2 adversarial spill-over adversarial samples which have switched cluster labels from the Rosa cluster (or cluster $k_1$) to the Kama (or cluster $k_2$) cluster including the target sample. The original clustering with the target sample selected is shown in Fig. 6, and is plotted in 3D using the area, perimeter, and compactness features. Here $\mathcal{D}_M(x_t') = 0.33$ which is the 0.28 quantile of $\mathcal{D}_M(X)$ which ensures that it will not be an outlier.
\begin{figure}[t]
\centering
\includegraphics[scale=0.38]{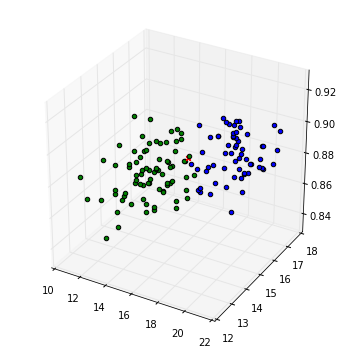}
\caption{Kama and Rosa wheat kernel clusters (target sample to be adversarially perturbed in red) visualized using the \textit{area}, \textit{perimeter}, and \textit{compactness} wheat seed features}
\end{figure}
\subsubsection{MoCap Hand Postures Dataset}
 The MoCap Hand Postures dataset \cite{gardner2014measuring} consists of 5 types of hand postures/gestures from 12 users recorded in a motion capture environment using 11 unlabeled markers attached to a glove. We only use a small subset of the data with 200 samples for each cluster. For clustering, the possible features are each of the 11 markers' $X, Y, Z$ coordinates. However, we only use the first 3 markers' recorded $X, Y, Z$ coordinates because due to resolution and occlusion, missing values are common in the other markers' data. Thus, we have a total of 9 features: $X_i, Y_i, Z_i$ for each $i^{th}$ marker, where $i=1,..,3$.
We use the K-Means clustering for clustering the Point1 posture and the Grab posture. Cluster size for Grab posture is $n_1 = 209$, and for Point1 posture is $n_2 = 191$. Algorithm 1 starts with the target sample in the Grab posture cluster, and generates 5 adversarial spill-over adversarial samples which have switched cluster labels from the Grab cluster ($k_1$) to the Point1 ($k_2$) cluster including the target sample. The original clustering with the target sample selected is shown in Fig. 7, and is plotted in 3D using the $Z_1, Z_2, Z_3$ marker coordinates features. Here $\mathcal{D}_M(x_t') = 0.325$ is the 0.27 quantile of $\mathcal{D}_M(X)$. This indicates that $x_t'$ is not an outlier.
\begin{figure}[t]
\centering
\includegraphics[scale=0.38]{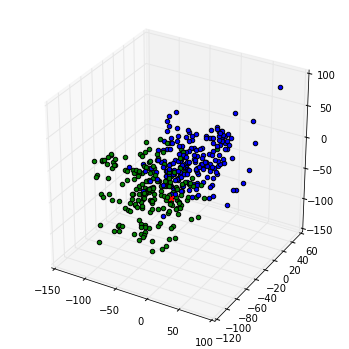}
\caption{Point1 and Grab pose clusters (target sample to be adversarially perturbed in red) visualized using the $Z_1, Z_2, Z_3$ marker position features}
\end{figure}
\section{Conclusion and Future Work}
In this paper, we propose a black-box adversarial attack algorithm which shows that clustering algorithms are vulnerable to adversarial attack even against a fairly constrained adversary. Our contributions are as follows:
\begin{itemize}
    \item Our attack algorithm creates \textit{spill-over} adversarial samples by perturbing one sample, and consequently perturbing the decision boundary between clusters. To the best of our knowledge, this is the first work in the field which generates additional adversarial samples without adding noise to those samples. (Section 3)
    \item We provide theoretical justification for existence of spill-over adversarial samples for K-Means clustering. We believe this is the first theoretical result showing the existence of spill-over adversarial samples (Section 4).
    \item We theoretically show that misclustering can happen using the spill-over adversarial attack even when the attacker does not have access to the true metric used for the clustering, but uses a noisy metric to cluster the data. This makes our attack especially powerful (Section 4).
    \item Our attack algorithm allows for the adversary to choose the noise threhsold $\Delta$ such that the perturbed adversarial sample does not become an outlier. We accomplish this by proposing the notion of Mahalanobis Depth for Clusters
    ($\mathcal{D}_C$), and Coordinatewise-Min-Mahalanobis Depth for high dimensional clustered data ($\mathcal{D}_M$). 
    \item We test the attack algorithm on Ward's Hierarchical clustering, and the K-Means clustering on a number of datasets, e.g., the UCI Handwritten Digits dataset, the MNIST dataset, the MoCap Hand Postures dataset, and the UCI Wheat Seeds dataset. We successfully carry out adversarial attacks on the clustering algorithms for all datasets even though true metric is unknown (Section 5). 
\end{itemize}
In future, we will provide theoretical results for Ward's Hierarchical clustering, since empirically the attack algorithm is successful in generating adversarial samples against it. Improving the optimization approach, and substituting the black-box heuristic optimization approaches with a more robust procedure that has provable convergence guarantees is also a promising direction. Much like supervised learning algorithms, we find that clustering algorithms are also vulnerable to powerful black-box adversarial attacks, making it imperative to design robust clustering approaches.
\bibliographystyle{aaai.bst}
\bibliography{aaai}

\begin{thebibliography}{}

\bibitem[\protect\citeauthoryear{Alpaydin and Kaynak}{1995}]{uci_digits}
Alpaydin, E., and Kaynak, C.
\newblock 1995.
\newblock {Optical Recognition of Handwritten Digits}.
\newblock \textit{http://archive.ics.uci.edu/ml/datasets/
  Optical+Recognition+of+Handwritten+Digits}.
\newblock [Online; accessed 1-November-2018].

\bibitem[\protect\citeauthoryear{Biggio \bgroup et al\mbox.\egroup
  }{2013}]{biggio2013data}
Biggio, B.; Pillai, I.; Rota~Bul{\`o}, S.; Ariu, D.; Pelillo, M.; and Roli, F.
\newblock 2013.
\newblock Is data clustering in adversarial settings secure?
\newblock In {\em Proceedings of the 2013 ACM workshop on Artificial
  intelligence and security},  87--98.
\newblock ACM.

\bibitem[\protect\citeauthoryear{Biggio \bgroup et al\mbox.\egroup
  }{2014}]{biggio2014poisoning}
Biggio, B.; Bul{\`o}, S.~R.; Pillai, I.; Mura, M.; Mequanint, E.~Z.; Pelillo,
  M.; and Roli, F.
\newblock 2014.
\newblock Poisoning complete-linkage hierarchical clustering.
\newblock In {\em Joint IAPR International Workshops on Statistical Techniques
  in Pattern Recognition (SPR) and Structural and Syntactic Pattern Recognition
  (SSPR)},  42--52.
\newblock Springer.

\bibitem[\protect\citeauthoryear{Carlini and
  Wagner}{2017}]{Carlini:2017:AEE:3128572.3140444}
Carlini, N., and Wagner, D.
\newblock 2017.
\newblock Adversarial examples are not easily detected: Bypassing ten detection
  methods.
\newblock In {\em Proceedings of the 10th ACM Workshop on Artificial
  Intelligence and Security}, AISec '17,  3--14.
\newblock New York, NY, USA: ACM.

\bibitem[\protect\citeauthoryear{Charytanowicz \bgroup et al\mbox.\egroup
  }{2010}]{charytanowicz2010complete}
Charytanowicz, M.; Niewczas, J.; Kulczycki, P.; Kowalski, P.~A.; {\L}ukasik,
  S.; and {\.Z}ak, S.
\newblock 2010.
\newblock Complete gradient clustering algorithm for features analysis of x-ray
  images.
\newblock In {\em Information technologies in biomedicine}. Springer.
\newblock  15--24.

\bibitem[\protect\citeauthoryear{Chhabra}{2019}]{code}
Chhabra, A.
\newblock 2019.
\newblock {}.
\newblock \textit{https://github.com/anshuman23/aaai-adversarial-clustering}.
\newblock [Online; accessed 1-November-2019].

\bibitem[\protect\citeauthoryear{Crussell and
  Kegelmeyer}{2015}]{crussell2015attacking}
Crussell, J., and Kegelmeyer, P.
\newblock 2015.
\newblock Attacking dbscan for fun and profit.
\newblock In {\em Proceedings of the 2015 SIAM International Conference on Data
  Mining},  235--243.
\newblock SIAM.

\bibitem[\protect\citeauthoryear{Crussell, Gibler, and
  Chen}{2013}]{crussell2013andarwin}
Crussell, J.; Gibler, C.; and Chen, H.
\newblock 2013.
\newblock Andarwin: Scalable detection of semantically similar android
  applications.
\newblock In {\em European Symposium on Research in Computer Security},
  182--199.
\newblock Springer.

\bibitem[\protect\citeauthoryear{Dutrisac and
  Skillicorn}{2008}]{dutrisac2008hiding}
Dutrisac, J., and Skillicorn, D.~B.
\newblock 2008.
\newblock Hiding clusters in adversarial settings.
\newblock In {\em Intelligence and Security Informatics, 2008. ISI 2008. IEEE
  International Conference on},  185--187.
\newblock IEEE.

\bibitem[\protect\citeauthoryear{Gardner \bgroup et al\mbox.\egroup
  }{2014}]{gardner2014measuring}
Gardner, A.; Kanno, J.; Duncan, C.~A.; and Selmic, R.
\newblock 2014.
\newblock Measuring distance between unordered sets of different sizes.
\newblock In {\em Proceedings of the IEEE Conference on Computer Vision and
  Pattern Recognition},  137--143.

\bibitem[\protect\citeauthoryear{Goldberg}{2006}]{goldberg2006genetic}
Goldberg, D.~E.
\newblock 2006.
\newblock {\em Genetic algorithms}.
\newblock Pearson Education India.

\bibitem[\protect\citeauthoryear{Holmstr{\"o}m, Quttineh, and
  Edvall}{2008}]{holmstrom2008adaptive}
Holmstr{\"o}m, K.; Quttineh, N.-H.; and Edvall, M.~M.
\newblock 2008.
\newblock An adaptive radial basis algorithm (arbf) for expensive black-box
  mixed-integer constrained global optimization.
\newblock {\em Optimization and Engineering} 9(4):311--339.

\bibitem[\protect\citeauthoryear{Kirkpatrick, Gelatt, and
  Vecchi}{1983}]{kirkpatrick1983optimization}
Kirkpatrick, S.; Gelatt, C.~D.; and Vecchi, M.~P.
\newblock 1983.
\newblock Optimization by simulated annealing.
\newblock {\em science} 220(4598):671--680.

\bibitem[\protect\citeauthoryear{Knysh and Korkolis}{2016}]{knysh2016blackbox}
Knysh, P., and Korkolis, Y.
\newblock 2016.
\newblock Blackbox: A procedure for parallel optimization of expensive
  black-box functions.
\newblock {\em arXiv preprint arXiv:1605.00998}.

\bibitem[\protect\citeauthoryear{Knysh}{2016}]{blackbox}
Knysh, P.
\newblock 2016.
\newblock \texttt{Blackbox}: A python module for parallel optimization of
  expensive black-box functions.
\newblock \textit{https://github.com/paulknysh/blackbox}.
\newblock [Online; accessed 1-August-2019].

\bibitem[\protect\citeauthoryear{LeCun}{1998}]{lecun1998mnist}
LeCun, Y.
\newblock 1998.
\newblock The mnist database of handwritten digits.
\newblock \textit{http://yann. lecun. com/exdb/mnist/}.
\newblock [Online; accessed 1-November-2018].

\bibitem[\protect\citeauthoryear{Lloyd}{1982}]{Lloyd:2006:LSQ:2263356.2269955}
Lloyd, S.
\newblock 1982.
\newblock Least squares quantization in pcm.
\newblock {\em IEEE Trans. Inf. Theor.} 28(2):129--137.

\bibitem[\protect\citeauthoryear{McKay, Beckman, and
  Conover}{2000}]{mckay2000comparison}
McKay, M.~D.; Beckman, R.~J.; and Conover, W.~J.
\newblock 2000.
\newblock A comparison of three methods for selecting values of input variables
  in the analysis of output from a computer code.
\newblock {\em Technometrics} 42(1):55--61.

\bibitem[\protect\citeauthoryear{Paindaveine and
  Van~Bever}{2013}]{paindaveine2013depth}
Paindaveine, D., and Van~Bever, G.
\newblock 2013.
\newblock From depth to local depth: a focus on centrality.
\newblock {\em Journal of the American Statistical Association}
  108(503):1105--1119.

\bibitem[\protect\citeauthoryear{Papernot, McDaniel, and
  Goodfellow}{2016}]{papernot2016transferability}
Papernot, N.; McDaniel, P.; and Goodfellow, I.
\newblock 2016.
\newblock Transferability in machine learning: from phenomena to black-box
  attacks using adversarial samples.
\newblock {\em arXiv preprint arXiv:1605.07277}.

\bibitem[\protect\citeauthoryear{Pedregosa \bgroup et al\mbox.\egroup
  }{2011}]{pedregosa2011scikit}
Pedregosa, F.; Varoquaux, G.; Gramfort, A.; Michel, V.; Thirion, B.; Grisel,
  O.; Blondel, M.; Prettenhofer, P.; Weiss, R.; Dubourg, V.; et~al.
\newblock 2011.
\newblock Scikit-learn: Machine learning in python.
\newblock {\em Journal of machine learning research} 12(Oct):2825--2830.

\bibitem[\protect\citeauthoryear{Perdisci, Ariu, and Giacinto}{2013}]{malware2}
Perdisci, R.; Ariu, D.; and Giacinto, G.
\newblock 2013.
\newblock Scalable fine-grained behavioral clustering of http-based malware.
\newblock {\em Computer Networks} 57(2):487 -- 500.
\newblock Botnet Activity: Analysis, Detection and Shutdown.

\bibitem[\protect\citeauthoryear{{P}ouget \bgroup et al\mbox.\egroup
  }{2006}]{malware1}
{P}ouget, F.; {D}acier, M.; {Z}immerman, J.; {C}lark, A.; and {M}ohay, G.
\newblock 2006.
\newblock {I}nternet attack knowledge discovery via clusters and cliques of
  attack traces.
\newblock {\em {J}ournal of {I}nformation {A}ssurance and {S}ecurity, {V}olume
  1, {I}ssue 1, {M}arch 2006}.

\bibitem[\protect\citeauthoryear{Regis and Shoemaker}{2005}]{Regis2005}
Regis, R.~G., and Shoemaker, C.~A.
\newblock 2005.
\newblock Constrained global optimization of expensive black box functions
  using radial basis functions.
\newblock {\em Journal of Global Optimization} 31(1):153--171.

\bibitem[\protect\citeauthoryear{Sander \bgroup et al\mbox.\egroup
  }{1998}]{sander1998density}
Sander, J.; Ester, M.; Kriegel, H.-P.; and Xu, X.
\newblock 1998.
\newblock Density-based clustering in spatial databases: The algorithm gdbscan
  and its applications.
\newblock {\em Data mining and knowledge discovery} 2(2):169--194.

\bibitem[\protect\citeauthoryear{Skillicorn}{2009}]{skillicorn2009adversarial}
Skillicorn, D.~B.
\newblock 2009.
\newblock Adversarial knowledge discovery.
\newblock {\em IEEE Intelligent Systems} 24:54--61.

\bibitem[\protect\citeauthoryear{Szegedy \bgroup et al\mbox.\egroup
  }{2013}]{szegedy2013intriguing}
Szegedy, C.; Zaremba, W.; Sutskever, I.; Bruna, J.; Erhan, D.; Goodfellow, I.;
  and Fergus, R.
\newblock 2013.
\newblock Intriguing properties of neural networks.
\newblock {\em arXiv preprint arXiv:1312.6199}.

\bibitem[\protect\citeauthoryear{Ward~Jr}{1963}]{ward1963hierarchical}
Ward~Jr, J.~H.
\newblock 1963.
\newblock Hierarchical grouping to optimize an objective function.
\newblock {\em Journal of the American statistical association}
  58(301):236--244.

\bibitem[\protect\citeauthoryear{Xing \bgroup et al\mbox.\egroup
  }{2003}]{xing2003distance}
Xing, E.~P.; Jordan, M.~I.; Russell, S.~J.; and Ng, A.~Y.
\newblock 2003.
\newblock Distance metric learning with application to clustering with
  side-information.
\newblock In {\em Advances in neural information processing systems},
  521--528.

\end{thebibliography}
\end{document}